%% file: main.tex
\documentclass[letterpaper, 10 pt,journal,twoside]{IEEEtran}
\usepackage[plain]{algorithm}
\usepackage[noend]{algorithmic}

\usepackage{amsmath}
\usepackage{amssymb}
\usepackage{amsthm}
\usepackage{graphicx}
\usepackage{subcaption}
\usepackage{multirow}
\usepackage{url}
\usepackage{mathtools}
\newtheorem{theorem}{Theorem}
\newtheorem{lemma}{Lemma}

\newtheorem{corollary}{Corollary}
\newtheorem*{azuma}{Maximal Azuma-Hoeffding inequality}
\DeclareMathOperator*{\argmax}{arg\,max}

\begin{document}

\author{Indu~John,
        Chandramouli~Kamanchi
        and~Shalabh~Bhatnagar
\thanks{The authors are from the Department of Computer Science and Automation at the Indian Institute of Science, Bangalore 560012. E-mails: $\{$indu, chandramouli, shalabh$\}$@iisc.ac.in.}%
\thanks{Shalabh Bhatnagar was supported in part with projects from the Robert Bosch Centre for Cyber Physical Systems and Department of Science and Technology, and the J.C.Bose Fellowship.}%
}
\pagestyle{empty}
\title{Generalized Speedy Q-learning}

\maketitle
\thispagestyle{empty}

\begin{abstract}
In this paper, we derive a generalization of the Speedy Q-learning (SQL) algorithm that was proposed in the Reinforcement Learning (RL) literature to handle slow convergence of Watkins' Q-learning. In most RL algorithms such as Q-learning, the Bellman equation and the Bellman operator play an important role. It is possible to generalize the Bellman operator using the technique of successive relaxation. We use the generalized Bellman operator to derive a simple and efficient family of algorithms called  Generalized Speedy Q-learning (GSQL-$w$) and analyze its finite time performance. We show that GSQL-$w$ has an improved finite time performance bound compared to SQL for the case when the relaxation parameter $w$ is greater than 1. 
This improvement is a consequence of the contraction factor of the generalized Bellman operator being less than that of the standard Bellman operator. Numerical experiments are provided to demonstrate the empirical performance of the GSQL-$w$ algorithm. 
\end{abstract}

\begin{IEEEkeywords}
Machine learning, Stochastic optimal control, Stochastic systems.
\end{IEEEkeywords}

\IEEEpeerreviewmaketitle

\input{Introduction.tex}
\input{Background.tex}
\input{Algorithm.tex}
\input{Analysis.tex}
\input{Experiments.tex}
\input{Conclusion.tex}

\ifCLASSOPTIONcaptionsoff
  \newpage
\fi
\bibliographystyle{IEEEtran}
\bibliography{References}

\end{document}

%% file: Introduction.tex
\section{Introduction}
\IEEEPARstart{R}{einforcement} Learning (RL) is a paradigm in which an agent operating in a dynamic environment learns the best action sequence or policy to take in order to achieve the desired outcome. The interaction between the agent and the environment is modelled as an infinite horizon discounted reward Markov Decision Process (MDP). Watkins' Q-learning \cite{watkins1992q} is one of the most popular reinforcement learning algorithms. It computes an estimate of the optimal state-action value function or the Q-function in each iteration. It is shown in \cite{watkins1992q} that the sequence of estimates converges to the Q-function asymptotically. The convergence rate is however slow \cite{even2003learning,szepesvari1998asymptotic}, especially when the discount factor is close to $1$.

Speedy Q-learning (SQL) was proposed in \cite{azar2011speedy} to address the issue of slow convergence of Q-learning. At each iteration, the SQL algorithm uses two successive estimates of the Q-function and an aggressive learning rate in its update rule. This enables SQL to achieve faster convergence and a superior finite time bound on performance as compared to Q-learning. 

The Q-function is the fixed point of the Q-Bellman operator. A technique known as successive relaxation can be applied to generalize the Bellman operator \cite{reetz1973solution} with an additional parameter $w$. In this paper, we introduce an algorithm that generalizes Speedy Q-learning by computing the fixed point of the generalized Bellman operator. It is known that the fixed point of the generalized Bellman operator also yields an optimal policy of the MDP \cite{reetz1973solution}. Our algorithm is named Generalized Speedy Q-learning (GSQL) and it has an associated relaxation parameter $w$. We analyze the finite time performance of the algorithm in a PAC ("Probably Approximately Correct") framework. Convergence of the algorithm is guaranteed for $0 < w \leq w^*$, where $w^*$ depends on the underlying MDP. It is shown that for values of $w$ greater than 1, GSQL-$w$ is superior to Speedy Q-learning. Thus, we have a generalization of Speedy Q-learning with better finite time performance bound. 

The key idea is as follows. Consider MDPs with the special structure that for every action in the action space, there is a positive probability of self-loop for every state in the state space. This structure can be exploited using the technique of successive relaxation on the Bellman operator. For MDPs with this structure, one can choose $w$ such that the finite time performance of the proposed algorithm is superior to that of SQL. We also show numerical experiments to confirm our theoretical assertions. 
\subsection{Related work}\label{relatedwork}
After Watkins introduced the original Q-learning algorithm, several variants of the same have been proposed with different properties. For example, $Q(\lambda)$ \cite{peng1994incremental} is a parameterized variant that uses the concept of eligibility traces. Double Q-learning \cite{hasselt2010double} and Speedy Q-learning \cite{azar2011speedy} use two estimates of the Q-function, for addressing the issues of over-estimation and slow convergence, respectively. A multi-timescale version of the Q-learning algorithm is presented in \cite{shalabh} and its convergence shown using a differential inclusions based analysis. More recently, the Zap Q-learning algorithm was introduced \cite{devraj2017zap}, which is a matrix-gain algorithm designed to optimize the asymptotic variance.

Relaxation methods are iterative methods for solving systems of equations. A popular method is successive over-relaxation (SOR). SOR technique has been applied previously to solve an MDP when the model information is completely known \cite{reetz1973solution} and also in the setting of model-free reinforcement learning \cite{kamanchi2019successive}. The latter algorithm is known as SOR Q-learning. 

Although asymptotic convergence has been established for most of these algorithms, finite time behaviour which is important in practical applications is analyzed only for a few of them like Watkin's Q-learning \cite{even2003learning}, SARSA \cite{zou2019finite} and Speedy Q-learning (SQL) \cite{azar2011speedy}.

\subsection{Our Contributions}\label{contributions}
\begin{itemize}
    \item We generalize the Speedy Q-learning algorithm using the concept of successive relaxation to derive the GSQL-$w$ algorithm.
    \item We analyze the finite time performance of the GSQL-$w$ algorithm.
    \item We show that the generalization yields better bounds in the case $w>1.$
    \item We compare the empirical performance of GSQL-$w$ against similar algorithms in the literature.
\end{itemize}

%% file: Background.tex
\section{Background}
An RL problem can be modelled mathematically using the framework of Markov Decision Processes as described below.
A Markov Decision Process (MDP) is a 5-tuple $(S,A,P,R,\gamma)$, where $S$ is the set of states, $A$ is the set of actions, $P(j|i,a)$ is the transition probability from state $i$ to state $j$ when action $a$ is chosen, $R(i,a)$ is the reward obtained by taking action $a$ in state $i$ and $\gamma \in (0,1)$ is the discount factor. A policy $\pi$ is a mapping from states to actions. The goal is to find an optimal policy i.e., one that maximizes over all policies $\pi$ the expected long term discounted cumulative reward or value function given by \[ V^{\pi}(i)=\mathbb{E}\bigg{[}\sum_{t=0}^\infty \gamma^t R_t|s_0 =i\bigg{]}, \] where $s_0$ is the initial state and $R_t$ is the possibly random reward obtained at time $t$ with expected value $R(i,a)$ if the state at time $t$ is $i$ and the action chosen is $a$.

When the MDP model is completely known to the agent, numerical techniques such as value iteration and policy iteration are used to compute the optimal policy \cite{bertsekas1996neuro}. On the other hand, model-free reinforcement learning deals with the case where the agent learns to improve its behaviour based on its history of interactions with the environment. The agent does not have access to the full model, but has to learn from samples of the form 
$(s_{t}, a_{t}, R_{t}, s_{t+1})^{\infty}_{t=0}$ where $s_t$ is the current state at time $t$, $a_{t}$ is the action taken at time $t$ and $s_{t+1}$ is the next state observed after obtaining the reward $R_{t}$.

We assume that $S$ and $A$ are finite sets and the rewards $R(i,a)$ are all bounded by $R_{\max}$. Let $\beta \coloneqq \frac{1}{1-\gamma}$. Then, the long term discounted cumulative reward or value function is bounded by $V_{\max} \coloneqq \beta R_{\max}$. 

The generalized Bellman operator $H^w : \mathbb{R}^{S \times A} \rightarrow \mathbb{R}^{S \times A}$ is defined as \cite{kamanchi2019successive}
\begin{multline} \label{eq:gbo}
    (H^w Q)(i,a) := w(R(i,a) + \gamma \sum_{j \in S} P(j|i,a)\max_{b \in A} Q(j,b)) \\
    + (1-w) \max_{c \in A} Q(i,c) \text{ for } 0 < w \leq w^*.
\end{multline}
where $w^*\coloneqq \displaystyle\min_{i,a} \frac{1}{1-\gamma P(i|i,a)}$ (note that $w^* \geq 1$).

Let $Q^*$ be the unique fixed point  of $H^w$. The algorithm presented in the next section computes $Q^*$ iteratively, starting from an initial estimate $Q_0$. 
$Q^*$ is known \cite{kamanchi2019successive} to be such that an optimal policy $\pi^*$ is given by
\begin{align*}
  \pi^*(i)=\argmax_{a \in A} Q^*(i,a), \hspace{0.2cm}\forall i \in S,  
\end{align*}
and the corresponding optimal value function is given by
\begin{align*}
  V^*(i)=\max_{a \in A} Q^*(i,a), \hspace{0.2cm} \forall i \in S. 
\end{align*}

It is also proven, see \cite{kamanchi2019successive}, that $H^w$ is a max-norm contraction with contraction factor $(1-w+\gamma w)$, i.e., for $w \in (0,w^*]$ and $\gamma \in (0,1)$, it is shown that $1-w+\gamma w \in [0,1)$ and
\begin{equation*}
\|H^w P-H^w Q\| \leq (1-w+\gamma w) \|P-Q\|.
\end{equation*}
Throughout this paper, the $\|\cdot\|$ symbol is used to denote the max-norm, which is defined for a vector $x=(x_1,x_2,...,x_n)$ as $\|x\| \coloneqq \max(|x_1|,|x_2|,...,|x_n|)$.



\subsection{Speedy Q-learning}
The Speedy Q-learning algorithm also computes a state-action value function iteratively, according to the following update rule.
\begin{multline*}Q'_{n+1}(i,a) := Q'_n(i,a)+\alpha_n\big{(}H_n Q'_{n-1}(i,a)-Q'_n(i,a)\big{)}\\
+(1-\alpha_n)\big{(}H_n Q'_{n}(i,a)-H_n Q'_{n-1}(i,a)\big{)},
\end{multline*}
where $$H_n Q'_{n}(i,a) \coloneqq R(i,a)+ \gamma \displaystyle \max_{a \in A}Q'_{n}(j_n,a),$$ $$H_n Q'_{n-1}(i,a) \coloneqq R(i,a)+ \gamma \displaystyle \max_{a \in A}Q'_{n-1}(j_n,a),$$ $\alpha_n$ is the step-size and $j_n \sim P(\cdot|i,a)$.
It may be noted that the function $Q^*$ to which our algorithm converges could be different from the function $Q'$ to which Watkins' Q-learning or Speedy Q-learning algorithms converge $\big{(}$which is the fixed point of the Q-Bellman operator $H$ defined by $HQ(i,a) \coloneqq R(i,a)+\gamma \sum_{j \in S} P(j|i,a)\max_{b \in A} Q(j,b)\big{)}$. However, it has been established in \cite{kamanchi2019successive} that
\begin{equation}
    \max_{a \in A} Q^*(i,a) = \max_{a \in A} Q'(i,a), \hspace{0.2cm} \forall i \in S
\end{equation}
which shows that the same optimal value function is obtained from both $Q^*$ and $Q'$. 

%% file: Algorithm.tex
\section{Generalized Speedy Q-learning}
In this section, we present our algorithm that we call Generalized Speedy Q-learning (GSQL). The algorithm integrates ideas from Speedy Q-learning and Generalized Bellman operator given by equation \eqref{eq:gbo} in its update rule. In addition to the initial state-action value function $Q_0$ and the discount factor $\gamma$, the algorithm takes as input a parameter $w \in (0,w^*]$ which we refer as the relaxation parameter. 
\subsection{Algorithm} \label{sec:algo}
The pseudo-code of the synchronous version of the algorithm is given in Algorithm \ref{alg:gsql}. The term `synchronous' means that the Q-values corresponding to all (state, action) pairs are updated in every iteration by generating next-state samples from the transition matrix $P$. The advantage of synchronous version is that it simplifies the analysis of the algorithm.

Before describing the algorithm, we define an auxiliary transition probability rule $\mu$ as follows.
 \begin{equation} \label{mudef}
    \mu(j|i,a)= 
\begin{cases}
    \frac{\gamma w P(j|i,a)}{1-w+\gamma w },& \text{if } j \neq i \vspace{0.1cm}\\ 
    \frac{1-w+\gamma w P(i|i,a)}{1-w+\gamma w },            &\text{if } j=i.
\end{cases}
\end{equation}
Note that the choice of $w \in (0, w^*]$ ensures that $\mu(\cdot|i,a)$ is a probability mass function.  
\begin{algorithm}
 \caption{Generalized Speedy Q-learning (GSQL-$w$)}
 \label{alg:gsql}
 \begin{algorithmic}[1]
 \renewcommand{\algorithmicrequire}{\textbf{Input:}}
 \vspace{0.1cm}
 \REQUIRE Initial action-value function $Q_0$, discount factor $\gamma$, relaxation parameter $w$, number of iterations $N$
 \STATE $Q_{-1}=Q_0$
  \FOR {$n=0,1,2,...,N-1$}
  \STATE $\alpha_n=\frac{1}{n+1}$
  \FOR{each $(i,a) \in S \times A$}
    \STATE Generate modified next-state sample $j_n' \sim \mu(\cdot | i,a)$ using the sample $j_n \sim P(\cdot | i,a)$ 
  \STATE $H_n^w Q_{n-1}(i,a)$ \\ \hspace{0.5cm}  $\coloneqq wR(i,a)+(1-w+\gamma w) \displaystyle \max_{a \in A}Q_{n-1}(j_n',a)$
  \STATE $H_n^w Q_{n}(i,a)$ \\
  \hspace{0.5cm} $\coloneqq wR(i,a) + (1-w+\gamma w) \displaystyle \max_{a \in A}Q_{n}(j_n',a)$
  \STATE $Q_{n+1}(i,a)$\\ \hspace{0.5cm}
  $\coloneqq Q_n(i,a)+\alpha_n\big{(}H_n^w Q_{n-1}(i,a)-Q_n(i,a)\big{)}$\\
  \hspace{1cm}$+(1-\alpha_n)\big{(}H_n^w Q_{n}(i,a)-H_n^w Q_{n-1}(i,a)\big{)}$
  \ENDFOR
  \ENDFOR
 \RETURN $Q_N$ 
 \end{algorithmic} 
\end{algorithm}

\textit{Remark.} Given a sample from $P(\cdot|i,a)$, it is possible to generate a sample from $\mu(\cdot|i,a)$. For $0<w\leq1$, acceptance-rejection sampling \cite{robert2013monte} can be used. When $w>1$, the techniques developed in \cite{nacu2005fast,mossel2005new} are applicable. We see that \cite{nacu2005fast} discusses a fast simulation algorithm to generate samples when there are two states and  \cite{mossel2005new} generalizes to multiple states.


The update rule of GSQL-$w$ involves two successive estimates of the Q function, similar to Speedy Q-learning. The key differences are (i) the generation of a modified next-state sample $j_n'$ instead of $j_n$ in Step 5 and (ii) the generalized empirical Bellman operator $H_n^w$ instead of $H_n$ in Steps 6 and 7. These steps ensure that the expected value of the empirical operator $H^w_n$ is equal to the generalized Bellman operator $H^w$ as formally proved in Section \ref{analysis}. Since the contraction factor of the Generalized Bellman operator is less than that of the standard Bellman operator for $w>1$ (see Subsection \ref{CSQL}), the rate of convergence is faster in this case.
 \subsection{Finite time PAC performance bound}
The main theoretical result in this paper is a PAC bound on the performance of the Generalized Speedy Q-learning algorithm, which is as follows.
\begin{theorem}\label{theorem1}
    Let $Q_N$ be the state-action value function returned by the GSQL-$w$ algorithm after $N$ iterations. Then, with probability at least $1-\delta$,
    \begin{multline*}
     \|Q_N-Q^*\| \leq \\
     \frac{2(1-w+\gamma w)\beta^2 R_{\max}}{wN} + \frac{2 \beta^2 R_{\max}}{w} \sqrt{\frac{2\log\frac{2|S||A|}{\delta}}{N}}.
     \end{multline*}
\end{theorem}
The proof of Theorem \ref{theorem1} is given in Section \ref{analysis}.
\subsection{Comparison to Speedy Q-learning}\label{CSQL}
It is known that Speedy Q-learning \cite{azar2011speedy} has better sample complexity and computational complexity as compared to Watkins' Q-learning, with a space complexity of the same order. The finite time PAC bound  for Speedy Q-learning is as below.
$$ \|Q'_N-Q'\| \leq \frac{2\gamma \beta^2 R_{\max}}{N} + 2 \beta^2 R_{\max} \sqrt{\frac{2\log\frac{2|S||A|}{\delta}}{N}}.$$
There are two cases to consider, depending on the possible choice of $w$.
\begin{enumerate}
    \item $0 < w \leq 1$ (Under-relaxation) : \\
    In this case, $(1-w+\gamma w) \geq \gamma$ and $\frac{\beta}{w} \geq \beta$.
    \item $1 < w \leq w^*$ (Over-relaxation) : \\
    In this case, $(1-w+\gamma w) < \gamma$ and $\frac{\beta}{w} < \beta$. 
\end{enumerate}

It is seen that the bound for GSQL-$w$ is better than that of SQL for the case $w>1$ $\big{(}$The choice $w>1$ is allowed whenever $P(i|i,a)>0$ for all $(i,a)\big{)}$. For the second term in the bound, which is the dominating term, the improvement is by a factor of $w$. Moreover, the space complexity and computational complexity of our algorithm are the same as those of SQL.



%% file: Analysis.tex
\section{Theoretical analysis}
\label{analysis}
In this section, we provide a proof of Theorem \ref{theorem1}, which also implies convergence of the algorithm. 
 
To simplify the notation, let $\gamma_1 \coloneqq 1-w+\gamma w$ and $\beta_1 \coloneqq \frac{1}{1-\gamma_1}=\frac{\beta}{w}$. Recall that $\gamma_1 \in [0,1) \leq \gamma$ for $w \in [1,w^*]$. We define $(\mathcal{M}Q)(i)\coloneqq \max_{a \in A}Q(i,a)$.

The operators $H^w$ and $H^w_n$ were defined earlier. Define the operator $D_n[Q_n,Q_{n-1}]$ as
\begin{equation*} \label{eqstar}
    D_n[Q_n,Q_{n-1}](i,a) \coloneqq n H^w_n Q_n(i,a) - (n-1) H^w_n Q_{n-1}(i,a),
\end{equation*}
and $(\mathcal{M}D_n[Q_n,Q_{n-1}])(i)\coloneqq \max_{a \in A}D_n[Q_n,Q_{n-1}](i,a)$.
Let $\mathcal{F}_n$ be the $\sigma$-algebra generated by the sequence of random variables $\{j_1,j_2,...,j_n\}$. Observe that the sequence $\{\mathcal{F}_n\}$ is a filtration. We define the operator $D[Q_n,Q_{n-1}]$ as follows.
\begin{equation*}
    D[Q_n,Q_{n-1}](i,a) \coloneqq \mathbb{E}[D_n[Q_n,Q_{n-1}](i,a) | \mathcal{F}_{n-1} ].
\end{equation*}
The update rule of the GSQL algorithm can now be rewritten as below.
\begin{flalign}
    &Q_{n+1}(i,a) \nonumber \\
    &=(1-\alpha_n)Q_n(i,a) + \alpha_n D_n[Q_n,Q_{n-1}](i,a) \nonumber\\ 
    &=(1-\alpha_n)Q_n(i,a) + \alpha_n(D[Q_n,Q_{n-1}](i,a) - m_n(i,a)) \label{basecase}
\end{flalign}
where $m_n(i,a)\coloneqq D[Q_n,Q_{n-1}](i,a)-D_n[Q_n,Q_{n-1}](i,a)$.
Note that the sequence $\{m_n\}$ is a martingale difference sequence with respect to the filtration $\{\mathcal{F}_n\}$. Define 
\begin{equation*}
    M_n(i,a) \coloneqq \sum_{k=0}^n m_k(i,a).
\end{equation*}

Let $Q_n \coloneqq \left(Q_n(i,a),(i,a) \in S \times A\right)$. Similarly, let $R \coloneqq \left(R(i,a),(i,a) \in S \times A\right)$, $\mathcal{M}Q \coloneqq \left(\mathcal{M}Q(i), i \in S\right)$, $D_n[Q_n,Q_{n-1}] \coloneqq \left(D_n[Q_n,Q_{n-1}](i,a),(i,a) \in S \times A\right)$ and $M_n \coloneqq$ $\left(M_n(i,a), (i,a) \in S \times A\right)$.

We prove the theorem in the following steps. 
\begin{enumerate}
    \item \textbf{Lemma 1} shows that the expected value of the generalized empirical Bellman operator $H^w_n$ is equal to the generalized Bellman operator $H^w$.
    \item In \textbf{Lemma 2}, the update rule is rewritten in terms of the operator $H^w$ and an error term.
    \item \textbf{Lemma 3} provides a bound on $\|Q^*-Q_n\|$ in terms of a discounted sum of error terms $\{M_k\}_{k=0}^{n-1}$.
    \item We state the maximal Azuma-Hoeffding inequality for martingale difference sequences and apply it to bound $M_k$'s.
    \item Finally, by combining the steps above, we derive the finite time performance bound for Generalized Speedy Q-learning.
\end{enumerate}  
\begin{lemma}
    $\mathbb{E}[H_n^w Q_n(i,a)] = H^w Q_n(i,a)$.
\begin{proof}
    \begin{align} 
        & ~ \mathbb{E}[H_n^w Q_n(i,a)] \nonumber \\
        &=\mathbb{E}\left[wR(i,a)+\gamma_1 \mathcal{M}Q_n(j_n')\right] \nonumber \\
        &=wR(i,a)+\gamma_1 \sum_{j \in S} \mu(j|i,a)\mathcal{M}Q_n(j)\label{eq1}\\
        &=wR(i,a) + \gamma_1 \sum_{j \in S, j \neq i} \frac{\gamma w P(j|i,a)}{\gamma_1} \mathcal{M}Q_n(j) \nonumber \\
        & \hspace{3cm}+\gamma_1 \frac{1-w+\gamma w P(i|i,a)}{\gamma_1} \mathcal{M}Q_n(i) \nonumber \\ 
        &=wR(i,a)+\gamma w \sum_{j \in S}P(j|i,a)\mathcal{M}Q_n(j) +(1-w)\mathcal{M}Q_n(i) \nonumber\\
        &=H^w Q_n(i,a). \nonumber
        \end{align}
Here, equation \eqref{eq1} is obtained from the previous step using the fact that $j_n' \sim \mu$.
\end{proof}
\begin{corollary}
    \begin{equation*}
    D[Q_n,Q_{n-1}](i,a)= n H^w Q_n(i,a) - (n-1) H^w Q_{n-1}(i,a).
    \end{equation*}
\end{corollary}
\begin{proof}
Follows from the definitions of the operators $D,D_{n}$ and $H^w$.
\end{proof}
\end{lemma}

\begin{lemma}\label{lemma2}
    $Q_n=\frac{1}{n}(H^w Q_0 + (n-1) H^w Q_{n-1} - M_{n-1})$  $\forall n \geq 1$.
\begin{proof}
    We prove the result by induction. Recall that $\alpha_n=\frac{1}{n+1}$. The base case ($k=1$) is the same as Equation (\ref{basecase}). Let the result hold for $n$. Then, we can see that it holds for $n+1$, since,
    \begin{align}
        Q_{n+1} &= \nonumber \\
        &\frac{n}{n+1}Q_n +\frac{1}{n+1}(n H^w Q_n-(n-1)H^w Q_{n-1}-m_n) \label{step 1}\\
      &=\frac{n}{n+1}\left(\frac{1}{n}\left(H^w Q_0 + (n-1) H^w Q_{n-1} - M_{n-1}\right)\right)\nonumber \\
       &+\frac{1}{n+1}\left(n H^w Q_n-(n-1)H^w Q_{n-1}-m_n\right) \label{step 2}\\ \nonumber
       &=\frac{1}{n+1}(H^w Q_0 + n H^w Q_{n} - M_{n-1}-m_n)\nonumber \\
       &=\frac{1}{n+1}(H^w Q_0 + n H^w Q_{n} - M_n). \nonumber
    \end{align}
Note that equation \eqref{step 1} follows from  \eqref{basecase} and equation \eqref{step 2} is obtained from  \eqref{step 1} by utilizing induction hypothesis. Thus, the result holds for all $n \geq 1$.
\end{proof}

\end{lemma}

\begin{lemma}\label{lemma3}
    Assume that the initial state-action value function $Q_0$ is uniformly bounded by $V_{\max}$. Then, for all $n \geq 1$,
    \begin{equation*}
        \|Q^*-Q_n\| \leq \frac{2\gamma_1 \beta_1 V_{\max}}{n}+\frac{1}{n} \sum_{k=1}^{n}\gamma_1^{n-k}\|M_{k-1}\|.
    \end{equation*}
\begin{proof}
    We use induction to prove this result as well. The statement of the lemma holds for $n=1$, since,
    \begin{align*}
            & ~ \|Q^*-Q_1\| \\
            &= \|H^w Q^* - H^w_0 Q_0\| = \|H^w Q^* - H^w Q_0 + m_0\| \\
            &\leq\|H^w Q^* - H^w Q_0\| + \|m_0\| \leq \gamma_1 \|Q^* - Q_0\| + \|m_0\| \\
            & \leq 2\gamma_1 V_{\max} + \|m_0\| \leq 2\gamma_1 \beta_1 V_{\max} + \|M_0\|.
    \end{align*}
    Suppose the result holds for $n$. Then, using \textbf{Lemma \ref{lemma2}},
    \begin{align}
            & \|Q^*-Q_{n+1}\| \nonumber \\ 
            = & \bigg{\|}Q^*-\frac{1}{n+1}(H^w Q_0 + nH^w Q_{n} - M_{n})\bigg{\|} \label{line1} \\
            = & \bigg{\|}\frac{1}{n+1}\left(H^w Q^* -H^w Q_0\right) \nonumber \\ & \hspace{1cm}+\frac{n}{n+1}\left(H^w Q^* -H^w Q_n\right) + \frac{1}{n+1}M_n\bigg{\|} \label{line2} \\
            \leq &  \frac{\gamma_1}{n+1}\|Q^*-Q_0\| + \frac{n\gamma_1}{n+1}\|Q^*-Q_n\|+\frac{1}{n+1}\|M_n\| \label{line3}\\
            \leq & \frac{2\gamma_1}{n+1}V_{\max}+\frac{n\gamma_1}{n+1}\bigg{[}\frac{2\gamma_1 \beta_1 V_{\max}}{n} \nonumber\\ 
            & \hspace{2cm}+\frac{1}{n} \sum_{k=1}^{n}\gamma_1^{n-k}\|M_{k-1}\|\bigg{]} +\frac{1}{n+1}\|M_n\|\nonumber\\
            &=\frac{2\gamma_1 \beta_1 V_{\max}}{n+1}+\frac{1}{n+1} \sum_{k=1}^{n+1}\gamma_1^{n+1-k}\|M_{k-1}\|.\nonumber
    \end{align}
    Here equation \eqref{line2} is obtained from  \eqref{line1} using $H^{w}Q^*=Q^*$ and equation \eqref{line3} is obtained by noting that $H^{w}$ is a contraction.
    This proves the lemma for all $n \geq 1$.
\end{proof}
\end{lemma}
We use the following version \cite{cesa2006prediction} of the Maximal Azuma-Hoeffding inequality to derive a bound on the sequence $\|M_k\|$.
\begin{azuma}
    Let $V_1,V_2,...,$ be a martingale difference sequence with respect to some sequence $X_1,X_2,...,$ such that $V_i$ is uniformly bounded by $B$ $\forall i$. If $S_n=\sum_{k=1}^n V_k$, then for any $\epsilon >0$, 
    \begin{equation*}\label{azuma1}
        \mathrm{P}\left( \max_{1\leq n\leq N} S_n > \epsilon\right) \leq \exp\left(\frac{-\epsilon^2}{2NB^2}\right).
    \end{equation*}
\end{azuma}
To apply this result to the sequence $\{m_n(i,a)\}$, we first need to bound the terms in this sequence. This bound is obtained as a corollary of the following lemma.
\begin{lemma}
    Let $\|Q_0\|\leq V_{\max}$. Then $\|D_n[Q_n,Q_{n-1}]\|\leq V_{\max}$ $\forall n\geq 0$.
\begin{proof}
    We prove the lemma by induction.
    For $n=0$ we have,
    \begin{align*}
            & ~ \|D_0[Q_0,Q_{-1}]\| \\
            &\leq w\|R\|+\gamma_1\|\mathcal{M}Q_{-1}\| \leq w R_{\max}+(1-w+\gamma w)V_{\max}\\
            &=w(R_{\max}+\gamma V_{\max}) +(1-w) V_{\max} =V_{\max}.\\
    \end{align*}
Assume that the bound holds for $n$. Now,
\begin{align}
        & \|D_{n+1}[Q_{n+1},Q_{n}]\| \nonumber \\ 
        &\leq w\|R\|+\gamma_1\|(n+1)\mathcal{M}Q_{n+1}-n\mathcal{M}Q_n\| \nonumber \\
        &=w\|R\|+\gamma_1\bigg{\|}(n+1) \nonumber \\
        &\mathcal{M}\left(\frac{n}{n+1}Q_n+\frac{1}{n+1}D_n[Q_n,Q_{n-1}]\right)-n\mathcal{M}Q_n\bigg{\|} \nonumber\\
        &= w\|R\|+\gamma_1\bigg{\|} \mathcal{M}\left(nQ_n+D_n[Q_n,Q_{n-1}]\right)-\mathcal{M}\left(nQ_n\right)\bigg{\|} \nonumber\\
        & \leq w\|R\|+\gamma_1\bigg{\|} \mathcal{M}\left(nQ_n+D_n[Q_n,Q_{n-1}] - nQ_n\right)\bigg{\|}\label{stage2}\\
        & = w\|R\| + \gamma_1 \big{\|}\mathcal{M} D_n[Q_n,Q_{n-1}]\big{\|} \nonumber \\
        & \leq w\|R\| + \gamma_1 \big{\|}D_n[Q_n,Q_{n-1}]\big{\|} \label{stage4} \\
        &\leq w R_{\max}+\gamma_1 V_{\max} =V_{\max}\nonumber
\end{align}
Note the use of the inequality $\vert \max\{a\}-\max\{b\}\vert \leq \vert \max\{a-b\}\vert$ for all vectors $a,b$  
to obtain \eqref{stage2}, which is shown in \cite{kamanchi2019successive}. Equation \eqref{stage4} follows from the definition of $\mathcal{M}$. 
Thus, by induction, the bound holds for all $n\geq0$.
\end{proof}
\begin{corollary}
    $m_n(i,a)= \mathbb{E}\big{[}D_n[Q_n,Q_{n-1}](i,a)\vert \mathcal{F}_{n-1}\big{]}-D_n[Q_n,Q_{n-1}](i,a) \leq 2V_{\max}$ $\forall (i,a)$ and $\forall n\geq 0$.
\end{corollary}
\begin{corollary}[Stability of GSQL]
    $\|Q_n\|\leq V_{\max}$ $\forall n \geq 0$ as $Q_n = \frac{1}{n}\sum_{k=0}^{n-1}D_k[Q_k,Q_{k-1}]$.
\end{corollary}
\end{lemma}
\begin{proof}[\textbf{Proof of Theorem \ref{theorem1}}]
First, we bound the second term in the RHS of Lemma \ref{lemma3} as follows.
\begin{equation*}
    \begin{split}
        \frac{1}{n} \sum_{k=1}^{n}\gamma_1^{n-k}\|M_{k-1}\| &\leq \frac{1}{n} \sum_{k=1}^{n}\gamma_1^{n-k} \max_{1\leq k \leq n} \|M_{k-1}\|\\
        &\leq \frac{\beta_1 \displaystyle\max_{1\leq k \leq n}\|M_{k-1}\|}{n}.
    \end{split}
\end{equation*}
Now, we derive a bound on $\max_{1\leq k \leq n}\|M_{k-1}\|=\max_{(i,a)}\max_{1\leq k \leq n}|M_{k-1}(i,a)|$. For any $(i,a)$, we have,
\begin{align*}
        & \mathbb{P}\left(\max_{1 \leq k \leq n} \vert M_{k-1}(i,a)\vert > \epsilon\right)\\
    \leq & \mathbb{P}\left(\max_{1 \leq k \leq n}  M_{k-1}(i,a) > \epsilon\right)
        + \mathbb{P}\left(\max_{1 \leq k \leq n}  -M_{k-1}(i,a) > \epsilon\right).
\end{align*}
Bounding each term using Azuma-Hoeffding inequality, we get $\mathbb{P}\big{(}\max_{1 \leq k \leq n} \vert M_{k-1}(i,a)\vert > \epsilon\big{)} \leq
2\exp\left(\frac{-\epsilon^2}{8nV_{\max}^2}\right)$. 
Taking a union bound over the state-action space, we get
\begin{equation*}
    \mathbb{P}\left(\max_{1\leq k \leq n}\|M_{k-1}\|>\epsilon \right) \leq 2|S||A|\exp\left(\frac{-\epsilon^2}{8nV_{\max}^2}\right),
\end{equation*}
which can be rewritten as: For any $\delta >0$,
\begin{equation*}
    \mathbb{P}\left(\max_{1\leq k \leq n}\|M_{k-1}\|\leq V_{\max}\sqrt{8n\log\frac{2|S||A|}{\delta}}\right) \geq 1-\delta.
\end{equation*}
Thus, the result from Lemma \ref{lemma3} now becomes a high probability bound, as given below.
\begin{equation*}
    \|Q_n-Q^*\| \leq \frac{2\gamma_1\beta_1 V_{\max}}{n} + 2 \beta_1 V_{\max} \sqrt{\frac{2\log\frac{2|S||A|}{\delta}}{n}}
\end{equation*}
with probability at least $1-\delta$. Theorem \ref{theorem1} follows by taking $n=N$ and using the definitions of $\gamma_1, \beta_1$ and $V_{\max}$. 
\end{proof}

%% file: Experiments.tex
\addtolength{\textheight}{-2cm}
\section{Experiments}
In this section, we experimentally evaluate the performance of the GSQL-$w$ algorithm. We implement two variants of GSQL-$w$ described as follows. The first variant, which is denoted as GSQL1, is the GSQL-$w$ algorithm described in Algorithm \ref{alg:gsql}. The second variant, referred as GSQL2, avoids construction of the modified next-state sample (that needs to be obtained from fast simulation methods described in \cite{nacu2005fast,mossel2005new}) by using a different empirical Bellman operator, which is given by
\begin{multline}
     \bar{H}_n^w Q_{n}(i,a) := wR(i,a) + \gamma w \displaystyle \max_{a \in A}Q_{n}(j_n,a)\\
     +(1-w) \displaystyle \max_{b \in A}Q_{n}(i,b)
\end{multline}
This operator also has the same expected value as the generalized Bellman operator $H^w$ (See equation \eqref{eq:gbo}). All other steps of GSQL2 are the same as that of GSQL1. Note that the finite time bound for GSQL2 is unknown and finding this is an interesting research direction. 

First, we compare GSQL1 and GSQL2 with three other model-free lookup table based reinforcement learning algorithms namely Q-learning, Speedy Q-learning and Double Q-learning \cite{hasselt2010double} on randomly constructed MDPs. Next, we demonstrate the superiority of GSQL over SQL especially for MDPs with high values of discount factor and $P(i|i,a)$. The scalability of the proposed algorithm is evaluated by considering MDPs with state space cardinality as $10,50,100,500$ and $1000$. Next, we fix an MDP and show the comparison between different values of the relaxation parameter $w$ in the GSQL1 algorithm. Our implementation is available here\footnote{\url{https://github.com/indujohniisc/GSQL}}.

\begin{figure*}
	\centering
	\begin{subfigure}[b]{0.3\textwidth}
		\centering
		\includegraphics[width=\textwidth]{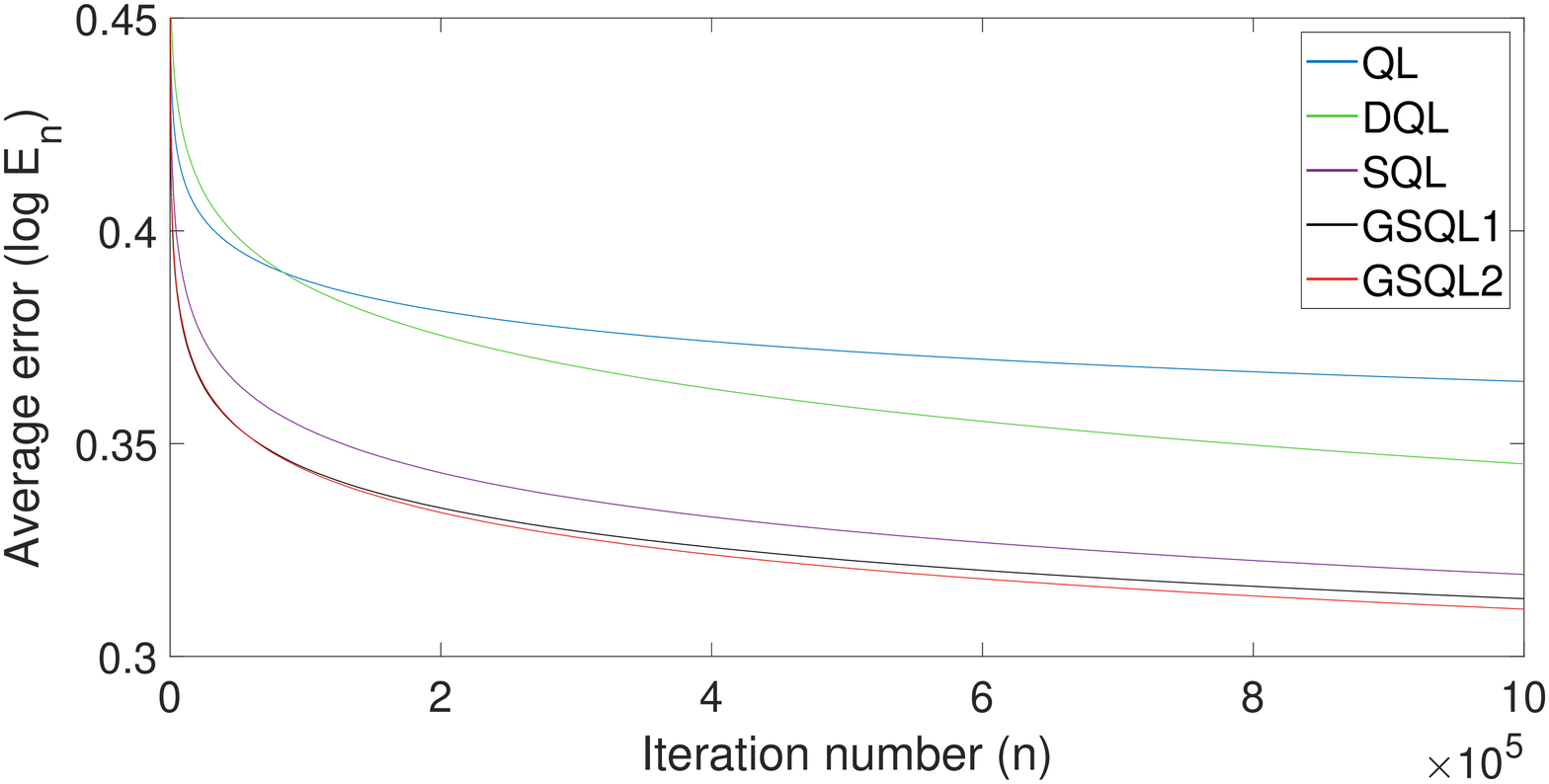}
		\caption{Average error of different RL algorithms with iterations}
		\label{fig:compalgos}
	\end{subfigure}
	\hfill
	\begin{subfigure}[b]{0.3\textwidth}
		\centering
		\includegraphics[width=\textwidth]{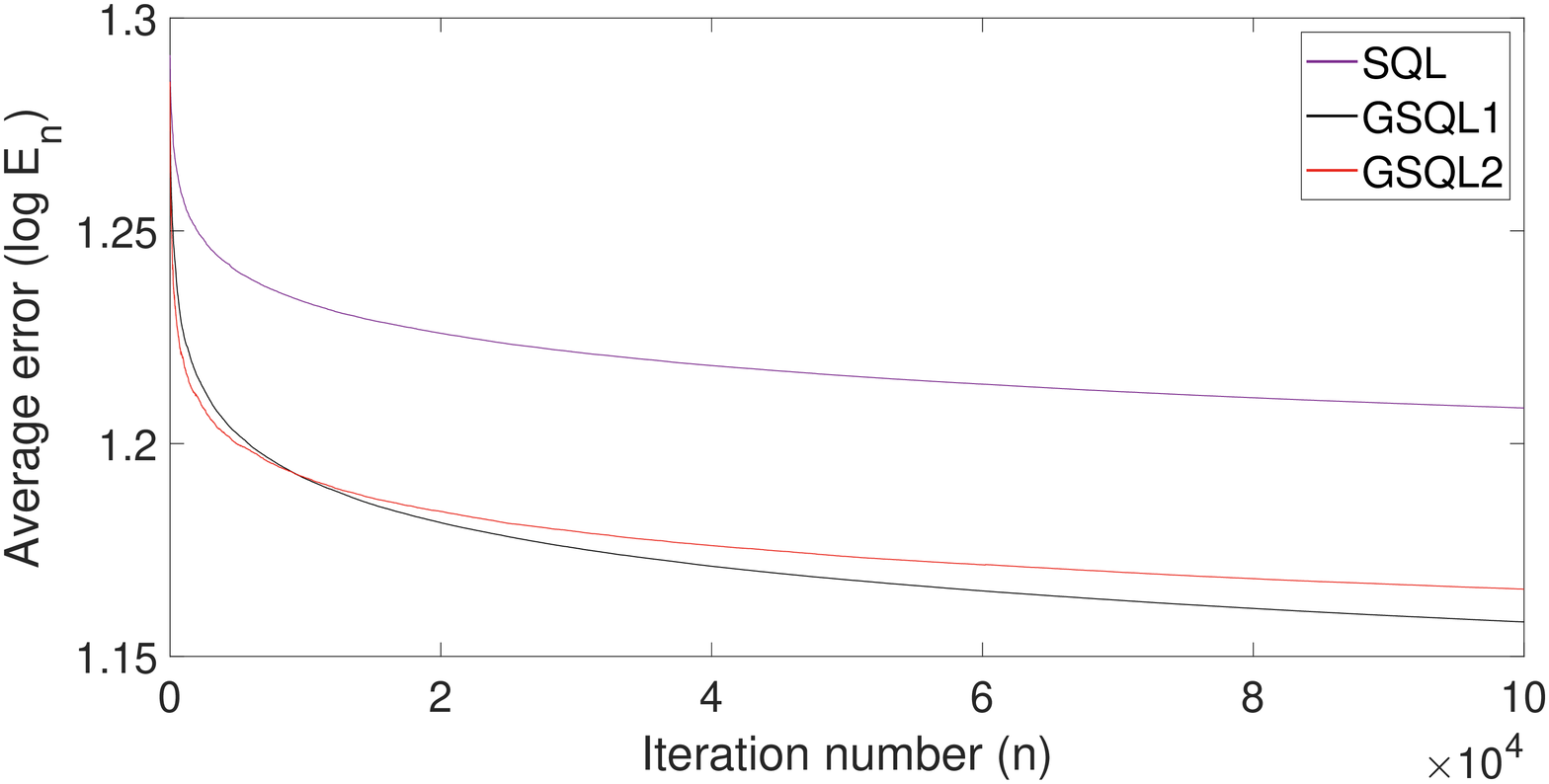}
		\caption{Average error of SQL vs GSQL when $w$ is large}
		\label{fig:compsql}
	\end{subfigure}
	\hfill
	\begin{subfigure}[b]{0.3\textwidth}
		\centering
		\includegraphics[width=\textwidth]{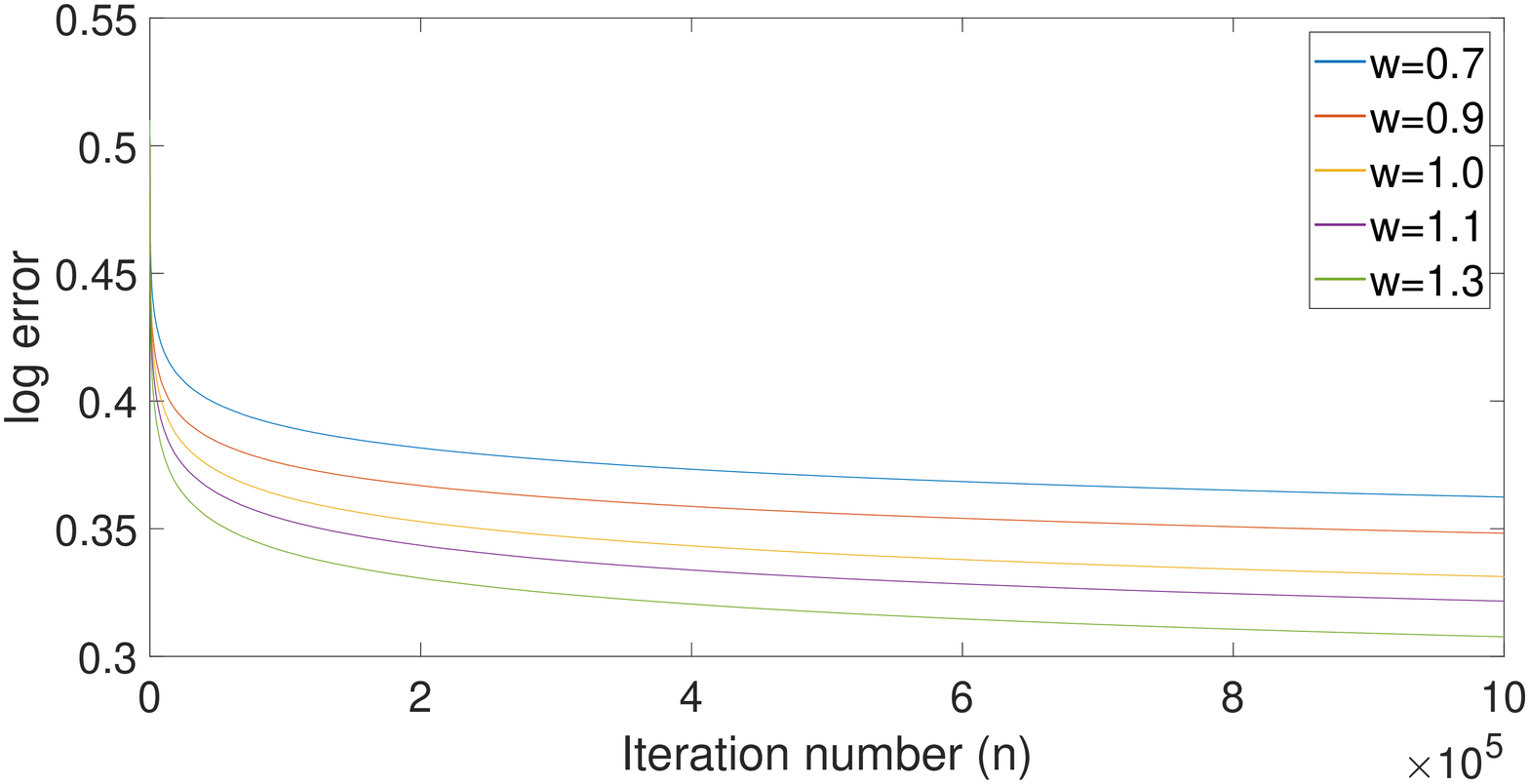}
		\caption{Error of GSQL for different values of $w$}
		\label{fig:compw}
	\end{subfigure}
	\caption{}
\end{figure*}
For comparing GSQL with other RL algorithms, we randomly generate $100$ MDPs with $10$ states and $5$ actions each, that satisfy the condition $P(i|i,a)>0$ $\forall (i,a)$ and have bounded rewards. We run Q-learning (QL), Speedy Q-learning (SQL), Double Q-learning (DQL) and the two variants of GSQL on these MDPs using the same initialization. The discount factor $\gamma$ is chosen as $0.6$. For GSQL1 and GSQL2, we set $w=w^*$ which gives the best finite time bound. 

Figure \ref{fig:compalgos} shows the average error for the different algorithms plotted against the iteration number. Average error is defined as the difference between the optimal value function and its estimate based on the current Q-function given by the algorithm averaged across the 100 MDPs, i.e.,
\begin{center}

  $ \displaystyle
    \begin{aligned} 
    E_n=\frac{1}{100} \sum_{k=1}^{100} \bigg{|}V_k^* -\max_a Q_{k,n}(\cdot,a)\bigg{|},
\end{aligned}
$ 
\end{center}
where $V_k^*$ is the optimal value function of the $k^{th}$ MDP and $Q_{k,n}$ is the Q-function estimate of the $k^{th}$ MDP at the $n^{th}$ iteration. It is seen that the average errors of GSQL1 and GSQL2 decrease with the number of iterations at a faster rate as compared to SQL and the other algorithms. This empirically shows that our algorithms work well for several different MDPs and their superiority over SQL. Further, both variants of GSQL have approximately the same error values which suggests that one or the other could be used in practice.

We also demonstrate the clear advantage of GSQL-$w$ over SQL when the relaxation parameter $w$ is large. For this, we have generated 10 random MDPs with discount factor 0.9 and $P(i|i,a) = 0.9$ $\forall (i,a)$ so that $w^*=5.26$. The result is shown in Figure \ref{fig:compsql}.
The scalability of the GSQL1 algorithm with the number of states is shown in Table \ref{tab:scalability}. The difference in error values of SQL and GSQL after $N(S)=|S| \times 10^3$ iterations is computed for MDPs with state space size $|S| = 10, 50, 100, 500$ and $1000$, for the same values of $w$ and $\gamma$. It is seen that the GSQL algorithm consistently outperforms SQL irrespective of the number of states.

\begin{table}[h]
\begin{tabular}{lcl}
\textbf{No of states}($|S|$) & \textbf{$E^{SQL}_{N(S)}-E^{GSQL}_{N(S)}$} & \textbf{Avg. time per iteration (s)}  \\ \hline
10                    & 1.6240                          & $9.31 \times 10^{-4}$                \\
50                    & 0.6364                         & $9.67 \times 10^{-4}$                \\
100                   & 0.7720                          & $11.99 \times 10^{-4}$               \\
500                   & 0.4487                          & $45.01 \times 10^{-4}$               \\
1000                  & 0.2981                          & $85.99 \times 10^{-4}$              
\end{tabular}
\caption{Scalability of GSQL algorithm}
\label{tab:scalability}
\end{table}
Next, we run the GSQL1 algorithm for different values of $w$ between $0$ and $w^*$ on a single MDP. The results are shown in Figure \ref{fig:compw}. As expected, higher values of $w$ show better performance.

%% file: Conclusion.tex
\section{Conclusion and Future Work}
This paper introduces a generalization of the Speedy Q-learning algorithm using the technique of successive relaxation and derives a PAC bound on its finite time performance. Different cases are discussed based on the value of the relaxation parameter $w$. The algorithm is designed to take advantage of the fact that the contraction factor of the generalized Bellman operator is less than that of the standard Bellman operator for the case $w>1$, so in this case, the bound obtained for the generalized algorithm is better than that of Speedy Q-learning.

The generalized Bellman operator can be used in other reinforcement learning algorithms as well. For example, it has already been applied to Watkins' Q-learning \cite{kamanchi2019successive}. It will be interesting to study the rate of convergence and other properties of the modified algorithms, both theoretically and experimentally. Another interesting direction would be to derive a function approximation version of GSQL to deal with the case of large state and action spaces.